\def\year{2018}\relax
\documentclass[letterpaper]{article} 
\usepackage{aaai18}  
\usepackage{times}  
\usepackage{helvet}  
\usepackage{courier}  
\usepackage{url}  
\usepackage{graphicx}  
\usepackage{bbm}
\usepackage{amsmath}
\usepackage{amsthm}
\usepackage{amssymb}
\usepackage{algorithm}
\usepackage{algorithmic}
\usepackage{caption}
\usepackage{subcaption}

\usepackage{multirow}
\usepackage[flushleft]{threeparttable}

\newtheorem{theorem}{Theorem}
\newtheorem{corollary}{Corollary}

\newtheorem{definition}{Definition}
\newtheorem{proposition}{Proposition}
\newtheorem{lemma}{Lemma}
\newtheorem{fact}{Fact}

\newcommand{\balpha}{\boldsymbol{\alpha}}
\newcommand{\bdelta}{\boldsymbol{\Delta}}
\newcommand{\bg}{\boldsymbol{g}}
\newcommand{\bh}{\boldsymbol{h}}
\newcommand{\bpi}{\boldsymbol{\pi}}
\newcommand{\bG}{\boldsymbol{G}}
\newcommand{\bY}{\boldsymbol{Y}}
\newcommand{\cK}{\mathcal{K}}
\newcommand{\cY}{\mathcal{Y}}
\newcommand{\cP}{\mathcal{P}}
\newcommand{\cE}{\mathcal{E}}
\newcommand{\cF}{\mathcal{F}}
\newcommand{\cN}{\mathcal{N}}
\newcommand{\bc}{\bar{c}}

\usepackage{color}
\newcommand{\blue}{}

\begin{document}
%
\title{Information Directed Sampling for Stochastic Bandits with Graph Feedback}
\author{Fang Liu \\
The Ohio State University\\
Columbus, Ohio 43210\\
liu.3977@osu.edu\\
\And Swapna Buccapatnam\\
AT\&T Labs Research\\
Middletown, NJ 07748\\
sb646f@att.com\\
\And Ness Shroff\\
The Ohio State University\\
Columbus, Ohio 43210\\
shroff.11@osu.edu\\
}
\maketitle
\begin{abstract}
We consider stochastic multi-armed bandit problems with graph feedback, where the decision maker is allowed to observe the neighboring actions of the chosen action. We allow the graph structure to vary with time and consider both deterministic and Erd{\H{o}}s-R{\'e}nyi random graph models. For such a graph feedback model, we first present a novel analysis of Thompson sampling that leads to tighter performance bound than existing work. Next, we propose new Information Directed Sampling based policies that are graph-aware in their decision making. Under the deterministic graph case, we establish a Bayesian regret bound for the proposed policies that scales with the clique cover number of the graph instead of the number of actions. Under the random graph case, we provide a Bayesian regret bound for the proposed policies that scales with the ratio of the number of actions over the expected number of observations per iteration. To the best of our knowledge, this is the first analytical result for stochastic bandits with random graph feedback. Finally, using numerical evaluations, we demonstrate that our proposed IDS policies outperform existing approaches, including adaptions of upper confidence bound, $\epsilon$-greedy and Exp3 algorithms.
  \end{abstract}

\section{Introduction}\label{sec:intro}
Multi-Armed Bandits (MAB) have been used as quintessential models for sequential decision making. In the classical MAB setting, at each time, a decision maker must choose an action from a set of $K$ actions with unknown probability distributions. Choosing an action $i$ at time $t$ reveals a random reward $X_{i}(t)$ drawn from the probability distribution of action $i.$  The goal is to find policies that minimize the expected loss due to uncertainty about actions' distributions over a given time horizon.      

In this work, we consider an important MAB setting, called the graph-structured feedback or the side-observation model~\cite{mannor,bhagat,sigmetrics2014,tossou2017thompson}, where choosing an action $i$ not only generates a reward from action $i$, but also reveals observations for a subset of the remaining actions. 

Such a scenario occurs in social networks, sensors networks, and advertising. For example, a decision maker must choose one user at each time in an online social network (e.g., Facebook) to offer a promotion~{\blue \cite{bhagat,carpentier2016revealing}}. Each time the decision maker offers a promotion to a user, he also has an opportunity to survey the user's neighbors in the network regarding their potential interest in a similar offer.\footnote{This is possible when the online network has an additional survey feature that generates ``side observations''. For example, when user $i$ is offered a promotion, her neighbors may be queried as follows: ``User $i$ was recently offered a promotion. Would you also be interested in the offer?".} Users are found to be more responsive to such surveys using social network information compared to generic surveys~\cite{ugander2011Facebook}, and this effect can be leveraged to construct side observations.

Consider another example, when the actions are advertisements~\cite{mannor} - the decision maker constructs a graph of different vacation places (Hawaii, Caribbean, Paris, etc.), where links capture similarities between different places. When a customer shows interest in one of the places, he is also asked to provide his opinion about the neighboring places in the graph.  

Side-observation models are also applicable to sensor networks that monitor events, where an agent must choose one (or a few) sensor to sample at each time. The reward obtained from choosing a sensor is related to its accuracy of monitoring the event. Neighboring sensors can communicate their observations to each other and this data aggregation has the desired affect of obtaining side-observations from neighboring sensors on determining which sensor(s) to select.

In~\cite{bhagat,sigmetrics2014,buccapatnam2017reward,tossou2017thompson}, the authors propose new policies for the side observation model in the stochastic bandit setting that exploit the graph structure to accelerate learning. In~\cite{bhagat,sigmetrics2014}, the authors propose extensions to upper confidence bound based policies, originally proposed for the classical MAB setting in~\cite{auer2002finite}. Policies proposed in~\cite{sigmetrics2014,buccapatnam2017reward}, namely $\epsilon_t$-greedy-LP and UCB-LP, are shown to be asymptotically optimal, both in terms of the network structure and time. In~\cite{tossou2017thompson}, the authors analyze the Bayesian regret performance of another well known classical bandit policy called Thompson Sampling (TS)~\cite{thompson} for the side-observation model. We make the following important contributions to this existing literature on graphical bandits:
\begin{itemize}
\item For the graph structured MAB feedback model we allow the side observation graph to vary with time. We focus on developing a problem-independent Bayesian regret bound. 
We provide a tighter bound for Thompson sampling, given in terms of the clique cover number of the side-observation graph than the bound presented in~\cite{tossou2017thompson}.
\item We also propose three algorithms, all of which are based on the approach of Information Directed Sampling (IDS) developed in~\cite{russo2014learning}. We show that these algorithms enjoy the same theoretical bounds as Thompson Sampling in terms of the clique cover number of the graph. However, using numerical evaluations, we show in Section~\ref{sec:simul}, that IDS based policies outperform existing policies in~\cite{sigmetrics2014}, that are provably asymptotically optimal both in terms of network structure and time. Hence, this raises the open question of how to determine better Bayesian regret bounds for our IDS based policies in terms of the network structure. 
\item  In contrast with existing works, we also consider the novel setting of a time variant random graph feedback model, where side observations from neighboring actions are obtained with a probability $r_t$ in time $t$.\footnote{This models the scenario of sensor networks where errors due to channel conditions can cause side observations to be randomly erased.} We provide Bayesian regret bounds for Thompson Sampling and our proposed IDS based policies for this probabilistic model as well. We believe that our work provides the first result for stochastic bandits with random graph feedback.
\end{itemize}

\section{Related Work}
Our work is related to~\cite{russo2014learning}, the authors of which propose a novel approach called IDS. IDS samples each action in a manner that minimizes the ratio between the squared expected single-period regret and the mutual information between optimal action and the next observation. It has been shown in~\cite{russo2014learning} using numerical simulations that IDS outperforms TS and Upper Confidence Bound (UCB)~\cite{auer2002finite} in the classical bandit setting. This motivated us to investigate extensions of IDS policy for the stochastic bandits with graph feedback and compare with adaptions of TS and UCB policies, which have been studied by~\cite{bhagat,sigmetrics2014,tossou2017thompson}. 


The Thompson Sampling algorithm as analyzed in~\cite{tossou2017thompson} does not explicitly use the graph structure in each step for its operation (similar to UCB-N algorithm proposed in~\cite{bhagat}). This is attractive when such graphical information is difficult to obtain, a case also studied in~\cite{Cohen2016}. However, in many cases such as the problem of promotions in online social networks and sensor networks, the graph structure is revealed or can be learned a priori. When such knowledge is available,~\cite{sigmetrics2014} show that using graphical information to make choices in each time helps in obtaining the optimal regret  both in terms of network structure and time asymptotically. This work motivated us to investigate extensions of IDS policies that exploit the knowledge of network structure in our work. Using numerical evaluations, we find that IDS policies outperform asymptotically optimal policies presented in~\cite{sigmetrics2014}.

{\blue Non-stochastic bandits with graph feedback have been studied by a line of work~\cite{mannor,alon2013bandits,kocak2014efficient,alon2014nonstochastic,alon2015online}. Other related partial feedback models include label efficient bandit in~\cite{audibert2010regret} and prediction with limited advice in~\cite{seldin2014prediction}, where side observations are limited by a budget. A summary of the bandits on graphs can be found in \cite{valko2016bandits}.}

\section{Problem Formulation}
\subsection{Stochastic Bandit Model}
We consider a Bayesian formulation of the stochastic $K$-armed bandit problem in which uncertainties are modeled as random variables. At each time $t\in\mathbb{N}$, a decision maker chooses an action $A_t$ from a finite action set $\cK=\{1,\ldots,K\}$ and receives the corresponding random reward $Y_{t,A_t}$. Without loss of generality, we assume the space of possible rewards $\cY=[0,1]$. Note that the results in this work can be extended to the case where reward distributions are sub-Gaussian. There is a random variable $Y_{t,a}\in\cY$ associated with each action $a\in\cK$ and $t\in\mathbb{N}$. We assume that $\{Y_{t,a},\forall a\in\cK\}$ are independent for each time $t$. Let $\bY_t\triangleq(Y_{t,a})_{a\in\cK}$ be the vector of random variables at time $t\in\mathbb{N}$. The true reward distribution $p^*$ is a distribution over $\cY^K$, which is randomly drawn from the family of distributions $\cP$ and unknown to the decision maker. Conditioned on $p^*$, $(\bY_t)_{t\in\mathbb{N}}$ is an independent and identically distributed sequence with each element $\bY_t$ sampled from the distribution $p^*$. 

Let $A^*\in \arg\max_{a\in\cK}\mathbb{E}[Y_{t,a}|p^*]$ be the true optimal action conditioned on $p^*$. Then the $T$ period regret of the decision maker is the expected difference between the total rewards obtained by an oracle that always chooses the optimal action and the accumulated rewards up to time horizon $T$. Formally, we study the expected regret
\begin{equation}\label{eqn:regret}
\mathbb{E}[R(T)]=\mathbb{E}\left[\sum_{t=1}^TY_{t,A^*}-Y_{t,A_t}\right],
\end{equation}
where the expectation is taken over the randomness in the action sequence $(A_1,\ldots,A_T)$ and the outcomes $(\bY_t)_{t\in\mathbb{N}}$ and over the prior distribution over $p^*$. This notion of regret is also known as \emph{Bayesian regret} or \emph{Bayes risk}.

\subsection{Graph Feedback Model}
In this problem, we assume the existence of side observations, which is described by a graph $G_t=(\cK,\cE_t)$ over the action set for each time $t$. The graph $G_t$ may be directed or undirected and can be dependent on time $t$. At each time $t$, the decision maker observes the reward $Y_{t,A_t}$ for playing action $A_t$ as well as the outcome $Y_{t,a}$ for each action $a\in\{a\in\cK|(A_t,a)\in\cE_t\}$. Note that it becomes the classical bandit feedback setting when the graph is empty (i.e., no edge exists) and it becomes the full-information (expert) setting when the graph is complete for all time $t$. In this work, we study two types of graph feedback models: \emph{deterministic graph} and \emph{Erd{\H{o}}s-R{\'e}nyi random graph}.

{\bf Deterministic graph.} In the deterministic graph feedback model, we assume that the graph $G_t$ is fixed before the decision is made at each time $t$. Let $\bG_t\in\mathbb{R}^{K\times K}$  be the adjacent matrix that represents the deterministic graph feedback structure $G_t$. Let $\bG_t(i,j)$ be the element at the $i$-th row and $j$-th column of the matrix. Then $\bG_t(i,j)=1$ if there exists an edge $(i,j)\in\cE_t$ and $\bG_t(i,j)=0$ otherwise. Note that we assume $\bG_t(i,i)=1$ for any $i\in\cK$. 
\begin{definition}\emph{(Clique cover number)
A \emph{Clique} of a graph $G=(\cK,\cE)$ is a subset $S\subseteq \cK$ such that the sub-graph formed by $S$ and $\cE$ is a complete graph.
A \emph{Clique cover} of a graph $G=(\cK,\cE)$ is a partition of $\cK$, denoted by $\mathcal{C}$, such that $S$ is a clique for each $S\in\mathcal{C}$. 
The cardinality of the smallest clique cover is called the \emph{clique cover number}, which is denoted by $\chi(G)$.
}
\end{definition}

In this work, we slightly abuse the notation of clique cover number and use $\chi(G_t)$ and $\chi(\bG_t)$ interchangeably since $\bG_t$ fully characterizes the graph structure $G_t$. 

{\bf Erd{\H{o}}s-R{\'e}nyi Random Graph.} In the Erd{\H{o}}s-R{\'e}nyi random graph feedback model, we assume that the graph $G_t$ is generated from an Erd{\H{o}}s-R{\'e}nyi model with time-dependent parameter $r_t$ after the decision is made at each time $t$. 
In other words, the decision maker can reveal the outcome $Y_{t,a}$ with probability $r_t$ for each action $a\neq A_t$ at time $t$. This feedback model is also known as \emph{probabilistically triggered arms}~\cite{chen2016combinatorial}. 

We generalize the adjacent matrix representation of a deterministic graph feedback model to a random graph feedback model, such that each $(i,j)$-th element of the matrix is the probability of observing action $j$ via playing action $i$. For each time $t$, the adjacent matrix is denoted by $\bG_t$ to unify the representation of our algorithms and analysis. Then, we have that $\bG_t(i,i)=1$ for any $i\in\cK$ and $\bG_t(i,j)=r_t$ for any $i\neq j$. Note that parameter $r_t$ fully characterizes the random graph feedback model.

%
%

\subsection{Randomized Policies}
We define all random variables with respect to a probability space $(\Omega,\mathcal{F},\mathbb{P})$. Consider the filtration $(\mathcal{F}_t)_{t\in\mathbb{N}}$ such that $\mathcal{F}_t\subseteq \mathcal{F}$ is the $\sigma$-algebra generated by the observation history $O_{t-1}$. The observation history $O_t$ includes all decisions, rewards and side observations from time $1$ to time $t$. For each time $t$, the decision maker chooses an action based on the history $O_{t-1}$ and possibly some randomness. Any policy of the decision maker can be viewed as a \emph{randomized policy} $\bpi$, which is an $\mathcal{F}_t$-adapted sequence $(\bpi_t)_{t\in\mathbb{N}}$. For each time $t$, the decision maker chooses an action randomly according to $\bpi_t(\cdot)=\mathbb{P}(A_t=\cdot|\mathcal{F}_t)$, which is a probability distribution over $\cK$. Let $\mathbb{E}[R(T,\bpi)]$ be the Bayesian regret defined by~(\ref{eqn:regret}) when the decisions $(A_1,\ldots,A_T)$ are chosen according to $\bpi$.

Uncertainty about $p^*$ induces uncertainty about the true optimal action $A^*$, which is described by a prior distribution $\balpha_1$ of $A^*$. Let $\balpha_t$ be the posterior distribution of $A^*$ given the history $O_{t-1}$, i.e., $\balpha_t(\cdot)=\mathbb{P}(A^*=\cdot|\mathcal{F}_t)$. Then, $\balpha_{t+1}$ can be updated by Bayes rule given $\balpha_{t}$, decision $A_t$, reward $Y_{t,A_t}$ and side observations. The \emph{Shannon entropy} of $\balpha_t$ is defined as $H(\balpha_t)\triangleq-\sum_{i\in\cK}\balpha_t(i)\log(\balpha_t(i))$. We slightly abuse the notion of $\bpi_t$ and $\balpha_t$ such that they represent distributions (or functions) over finite set $\cK$ as well as vectors in a simplex $\mathcal{S}\subset\mathbb{R}^K$. Note that $\mathcal{S}=\{\bpi\in\mathbb{R}^K|\sum_{i=1}^K\bpi(i)=1, \bpi(i)\geq 0, \forall i\in\cK\}$.

Let $\bdelta_t$ be the instantaneous regret vector such that the $i$-th coordinate, $\bdelta_t(i)\triangleq\mathbb{E}[Y_{t,A^*}-Y_{t,i}|\mathcal{F}_t]$, is the expected regret of playing action $i$ at time $t$. Let $\bg_t$ be the information gain vector such that the $i$-th coordinate, $\bg_t(i)=\mathbb{E}[H(\balpha_t)-H(\balpha_{t+1})|\mathcal{F}_t,A_t=i]$, is the expected information gain of playing action $i$ at time $t$. Note that the information gain of playing action $i$ consists of that of observing the reward $Y_{t,i}$ and possibly some side observations. We define the information gain of observing action $a$ (i.e., $Y_{t,a}$) as $\bh_t(a)\triangleq I_t(A^*;Y_{t,a})$, which is the \emph{mutual information} under the posterior distribution between random variables $A^*$ and $Y_{t,a}$. Let $D(\cdot ||\cdot)$ be the \emph{Kullback-Leibler} divergence between two distributions\footnote{If $P$ is absolutely continuous with respect to $Q$, then $D(P||Q)=\int \log\left(\frac{\text{d}P}{\text{d}Q}\right)\text{d}P$, where $\frac{\text{d}P}{\text{d}Q}$ is the Radon-Nikodym derivative of $P$ w.r.t. $Q$.}. By the definition of mutual information, we have that $I_t(A^*;Y_{t,a})\triangleq$
\begin{equation}
D(\mathbb{P}((A^*,Y_{t,a})\in\cdot|\mathcal{F}_t)||\mathbb{P}(A^*\in\cdot|\mathcal{F}_t)\mathbb{P}(Y_{t,a}\in\cdot|\mathcal{F}_t)).
\end{equation}
The following proposition reveals the relationship between vector $\bg_t$ and $\bh_t$.
\begin{proposition}\label{prop:infogain}
Under the (deterministic or random) graph feedback $\bG_t$, we have $\bg_t\geq\bG_t \bh_t.$
\end{proposition}
Intuitively, Proposition \ref{prop:infogain} shows that the information gain of observing the reward and some side observations is at least the sum of the information gain of each individual observation. A formal proof is provided in Appendix \ref{proof:infogain} in the supplemental material.

At each time $t$, a randomized policy updates $\balpha_t$, $\bdelta_t$ and $\bh_t$ and makes a decision according to a sampling distribution $\bpi_t$.

\section{Algorithms}
For any randomized policy, we define the \emph{information ratio} of sampling distribution $\bpi_t$ at time $t$ as 
\begin{equation}
\Psi_t(\bpi_t)\triangleq{(\bpi_t^T\bdelta_t)^2}/{(\bpi_t^T\bg_t)}.
\end{equation}
Note that $\bpi_t^T\bdelta_t$ is the expected instantaneous regret of the sampling distribution $\bpi_t$, and $\bpi_t^T\bg_t$ is the expected information gain of the sampling distribution $\bpi_t$. So the information ratio $\Psi_t(\bpi_t)$ measures the ``energy'' cost (which is the square of the expected instantaneous regret) per bit of information acquired. 

The key idea of the IDS based policy is keeping the information ratio bounded in order to balance between having low expected instantaneous regret (a.k.a. exploitation) and obtaining knowledge about the optimal action (a.k.a. exploration). In other words, if the information ratio is bounded, then the expected regret is bounded in terms of the maximum amount of information one could expect to acquire, which is at most the entropy of the prior distribution of $A^*$, i.e., $H(\balpha_1)$. As we show in Section \ref{sec:regret}, we can find upper bounds for the information ratios of the policies we provide here.

\begin{algorithm}[tb]
\caption{Meta-algorithm for Information Directed Sampling with Graph Feedback}
\label{alg:IDS}
\begin{algorithmic}
\REQUIRE Time horizon $T$ and feedback graph model $(\bG_t)_{t\leq T}$
\FOR{$t$ {\bfseries from} $1$ {\bfseries to} $T$}
\STATE{{\bf Updating statistics:} compute $\balpha_t$, $\bdelta_t$ and $\bh_t$ accordingly.}
\STATE{{\bf Generating policy:} generate $\bpi_t$ as a function of ($\balpha_t$, $\bdelta_t$, $\bh_t$, $\bG_t$). (To be determined)}
\STATE{{\bf Sampling:} sample $A_t$ according to $\bpi_t$, play action $A_t$ and receive reward $Y_{t,A_t}$.}
\STATE{{\bf Observations:} observe $Y_{t,a}$ if $(A_t,a)\in\cE_t$, where $G_t=(\cK,\cE_t)$ is the graph generated by $\bG_t$.}
\ENDFOR
\end{algorithmic}
\end{algorithm}

In practice, the information gain vector $\bg_t$ is quite complicated to compute even assuming a Bernoulli distribution model for each action. However, computing the information gain of observing each individual action, i.e., $\bh_t$, is much easier since it is only the mutual information of two random variables. By Proposition \ref{prop:infogain}, we have that $\Psi_t(\bpi_t)\leq{(\bpi_t^T\bdelta_t)^2}/{(\bpi_t^T\bG_t\bh_t)}$. So we can design our IDS based policies according to $\bh_t$ and $\bG_t$ instead of $\bg_t$. We provide a meta-algorithm for IDS based policies in Algorithm \ref{alg:IDS}. What remains is to design $\bpi_t$ as a function of $\balpha_t$, $\bdelta_t$, $\bh_t$ and $\bG_t$. Note that one can replace $\bG_t\bh_t$ by $\bg_t$ in the IDS based algorithms and the regret results in Section~\ref{sec:regret} still hold.

{\bf TS-N policy} is a natural adaption of Thompson Sampling under the graph feedback. It replaces the generating policy step in Algorithm \ref{alg:IDS} by
\begin{equation}
\bpi_t^{\text{TS-N}}=\balpha_t.
\end{equation}
The TS-N policy ignores the graph structure information $\bG_t$, and sample the action according to the posterior distribution of $A^*$. 

{\bf IDS-N policy} replaces the generating policy step in Algorithm \ref{alg:IDS} by $\bpi_t^{\text{IDS-N}}$, which is the solution of the following optimization problem $P_1$.
\begin{align}
P_{1}: \min_{\bpi_t\in\mathcal{S}} &\ {(\bpi_t^T\bdelta_t)^2}/{(\bpi_t^T\bG_t\bh_t)}.
\end{align}
The IDS-N policy greedily minimizes the information ratio (upper bound) at each time.

{\bf IDSN-LP policy}  replaces the generating policy step in Algorithm \ref{alg:IDS} by $\bpi_t^{\text{IDSN-LP}}$, which is the solution of the following linear programming problem $P_2$.
\begin{align}
P_{2}: \min_{\bpi_t\in\mathcal{S}} &\ \bpi_t^T\bdelta_t ~~~\mbox{s.t.} \ \bpi_t^T\bG_t\bh_t\geq\balpha_t^T\bG_t\bh_t.
\end{align}
The IDSN-LP policy greedily minimizes the expected instantaneous regret at each time with the constraint that the information gain is at least the one obtained by TS-N policy.

{\bf IDS-LP policy}  replaces the generating policy step in Algorithm \ref{alg:IDS} by $\bpi_t^{\text{IDS-LP}}$, which is the solution of the following linear programming problem $P_3$.
\begin{align}
P_{3}: \min_{\bpi_t\in\mathcal{S}} &\ \bpi_t^T\bdelta_t ~~~\mbox{s.t.} \ \bpi_t^T\bG_t\bh_t\geq\balpha_t^T\bh_t.
\end{align}
The IDS-LP policy greedily minimizes the expected instantaneous regret at each time with the constraint that the information gain is at least the one obtained by TS policy without graph feedback. IDS-LP policy reduces the extent of exploration compared to IDSN-LP policy. Intuitively, it greedily exploits the current knowledge of the optimal action with controlled exploration. Though we can not find better regret bound for IDS-LP than IDSN-LP, IDS-N and TS-N, IDS-LP outperforms the others in numerical results as shown in Section \ref{sec:simul}.

\section{Regret Analysis}\label{sec:regret}
In this section, we first present a known general bound for any randomized policy and provide the regret upper bound results of the proposed policies for the deterministic and random graph feedback.
The regret analysis relies on the following bound, which is shown in~\cite{journals/corr/RussoR14a}.
\begin{lemma}\label{prop:generalbound}
\emph{(General Bound from~\cite{journals/corr/RussoR14a})} For any policy $\bpi=(\bpi_1,\bpi_2,\bpi_3,\ldots)$ and time horizon $T\in \mathbb{N}$,
\begin{equation}
\mathbb{E}[R(T,\bpi)]\leq\sqrt{\sum_{t=1}^T\mathbb{E}_{\bpi}[\Psi_t(\bpi_t)]H(\balpha_1)}.
\end{equation}
\end{lemma}
Lemma~\ref{prop:generalbound} shows that we only need to bound expected information ratio $\mathbb{E}_{\bpi}[\Psi_t(\bpi_t)]$ to obtain an upper bound for a randomized policy. The next result follows from the fact that the information ratio of IDS-LP policy can be bounded by $K/2$.

\begin{theorem}\label{thm:IDSLP}
For any (deterministic or random) graph feedback, the Bayesian regret of IDS-LP is 
\begin{equation}\label{IDSLPbound}
\mathbb{E}[R(T,\bpi^{\text{IDS-LP}})]\leq\sqrt{\frac{K}{2}TH(\balpha_1)}.
\end{equation}
\end{theorem}
The key idea of the proof is comparing the information ratio of IDS-LP to that of TS with bandit feedback. The detailed proof of Theorem \ref{thm:IDSLP} can be found in Appendix \ref{proof:IDSLP}. The next proposition shows a general bound for information ratios of TS-N, IDS-N and IDSN-LP policies. 

\begin{proposition}\label{prop:ratiobound}
For any (deterministic or random) graph feedback $\bG_t$, we have that $\Psi_t(\bpi_t^{\text{TS-N}})$, $\Psi_t(\bpi_t^{\text{IDS-N}})$ and $\Psi_t(\bpi_t^{\text{IDSN-LP}})$ are upper-bounded by $\psi_t\triangleq\frac{\left(\bdelta_t^T\balpha_t\right)^2}{(\bG_t\bh_t)^T\balpha_t}$.
\end{proposition}
The proof of Proposition \ref{prop:ratiobound} can be found in Appendix \ref{proof:ratiobound}. Combining this result with Lemma \ref{prop:generalbound}, we can obtain unified regret result for TS-N, IDS-N and IDSN-LP by bounding the ratio $\psi_t$. Now, we are ready to present the regret results separately for the deterministic and the random graph feedback.

\subsection{Deterministic Graph}
The following result shows the unified regret upper bound of TS-N, IDS-N and IDSN-LP under the deterministic graph feedback. The detailed proof is presented in Apendix \ref{proof:TSNdeterm}.
\begin{theorem}\label{thm:TSNdeterm}
For any deterministic graph feedback $(\bG_1,\bG_2,\bG_3,\ldots)$, the Bayesian regrets of TS-N, IDS-N and IDSN-LP are upper-bounded by
\begin{equation}
\sqrt{\sum_{t=1}^T\frac{\chi(\bG_t)}{2}H(\balpha_1)}.
\end{equation}
\end{theorem}

Recently, a similar result for TS-N has been shown to be $\sqrt{\max_{t}\frac{\chi(\bG_t)}{2} TH(\balpha_1)}$ in~\cite{tossou2017thompson}. Apparently, Theorem~\ref{thm:TSNdeterm} provides a tighter bound. We have the following result when the graph is also time-invariant.
\begin{corollary}\label{cor:TSNdeterm}
For any time-invariant and deterministic graph feedback $\bG$ (i.e., $\bG_t=\bG$ $\forall t\in\mathbb{N}$), the Bayesian regrets of TS-N, IDS-N and IDSN-LP are upper-bounded by
\begin{equation}
\sqrt{\frac{\chi(\bG)}{2}TH(\balpha_1)}.
\end{equation}
\end{corollary}
{\blue
Corollary \ref{cor:TSNdeterm} shows that TS-N, IDS-N and IDSN-LP can benefit from the side observations. In other words, the above problem-independent regret upper bound scales with the clique cover number of the graph instead of the number of actions (it is known that the regret bound of TS without side observations is (\ref{IDSLPbound})). A similar result has been disclosed by Caron et al. \cite{bhagat} in the form of problem-dependent upper bound. They show that UCB-N scales with the clique cover number compared to UCB without side observations. 

An information theoretic lower bound on the problem-independent regret has been shown in \cite{mannor} to scale with the independence number\footnote{Independence number is the largest number of nodes without edges between them.}. In general, the independence number is less than or equal to the clique cover number. However, the equality holds for a large class of graphs, such as star graphs and perfect graphs. In other words, our policies are order-optimal for a large class of graphs.
}
\subsection{Erd{\H{o}}s-R{\'e}nyi Random Graph}
The following result shows the unified regret upper bound of TS-N, IDS-N and IDSN-LP under the Erd{\H{o}}s-R{\'e}nyi random graph feedback. The detailed proof is presented in Apendix \ref{proof:TSNrandom}.
\begin{theorem}\label{thm:TSNrandom}
For any random graph feedback $(r_1,r_2,r_3,\ldots)$, the Bayesian regrets of TS-N, IDS-N and IDSN-LP are upper-bounded by
\begin{equation}
\sqrt{\sum_{t=1}^T\frac{K}{2(Kr_t+1-r_t)}H(\balpha_1)}.
\end{equation}
\end{theorem}

As far as we know, this is the first result for stochastic bandit with random graph feedback. An analogous result has been shown for the non-stochastic bandit, for which Koc{\'a}k et al. \cite{kocak2016online} proposed Exp3-Res\footnote{Note that they assume that $r_t$ is not available to Exp3-Res. However, it is still reasonable to compare since TS-N is not aware of $r_t$ as well.} policy with guarantee of $O\left(\sqrt{\sum_{t=1}^T\frac{1}{r_t}\log K}\right)$ if $r_t\geq\frac{\log T}{2K-2}$ holds for all $t$. Theorem~\ref{thm:TSNrandom} recovers the same guarantee without restriction on $r_t$ since $H(\balpha_1)\leq \log K$. We have the following result when $r_t$ is time-invariant.


\begin{corollary}\label{cor:TSNrandom}
For any time-invariant and random graph feedback $r$ (i.e., $r_t=r$ $\forall t\in\mathbb{N}$), the Bayesian regrets of TS-N, IDS-N and IDSN-LP are upper-bounded by
\begin{equation}
\sqrt{\frac{K}{2(Kr+1-r)}TH(\balpha_1)}.
\end{equation}
\end{corollary}
Corollary \ref{cor:TSNrandom} shows that the benefit from side observations can be measured by the expected number of observations per time step, i.e., $(K-1)r+1$. In other words, the above regret upper bound scales with the ratio of the number of actions and the expected number of observations. When $r=1$, this ratio equals to 1, which yields the regret result for stochastic bandit with full information~\cite{russo2016information}. When $r=0$, this ratio equals to $K$, which yields the regret result for stochastic bandit with bandit feedback~\cite{russo2016information}. An analogous result can be found as Corollary 3 in \cite{alon2013bandits} for the non-stochastic bandits.

{\blue
\section{Computation}\label{sec:com}
In this section, we provide computational methods for updating statistics and discuss the complexity issues of the algorithms.

\subsection{Computational Methods for Updating Statistics}
Algorithm \ref{alg:IDS} offers an abstract design principle with the availability of the statistics (i.e., $\balpha_t$, $\bh_t$ and $\bdelta_t$). However, additional work is required to design efficient computational methods to update these statistics for specific problems. In general, the challenge of updating statistics is to compute and represent a posterior distribution given observations, which is also faced with Thompson Sampling. When the posterior distribution is complex, one can often generate samples from this distribution using Markov Chain Monte Carlo (MCMC) algorithms, enabling efficient implementation of IDS. A detailed discussion of applying MCMC methods for implementing randomized policy can be found in \cite{scott2010modern}. However, when the posterior distribution has a closed form or the conjugate prior is well studied, the posterior distributions can be efficiently computed and stored, as are the cases of Beta-Bernoulli bandits and Gaussian bandits~\cite{wu2015online}.

In the numerical experiment, we implement Algorithm 2 in \cite{russo2014learning} to represent the posterior distribution and compute the statistics\footnote{Note that the vector $\bg$ calculated in \cite{russo2014learning} is the vector $\bh$ in stochastic bandits with graph feedback.} for Beta-Bernoulli bandits. The key idea is that the Beta distribution is a conjugate prior for the Bernoulli distribution. Specifically, given the prior that the expectation $\theta_i$ is drawn from Beta$(\beta_i^1,\beta_i^2)$, the posterior distribution of observing $Y_i\sim$ Bernoulli$(\theta_i)$ is Beta$(\beta_i^1+Y_i,\beta_i^2+1-Y_i)$. So the posterior distribution can be updated and represented easily. Then what remains is to calculate the statistics $\balpha_t$, $\bh_t$ and $\bdelta_t$ given the posterior distributions. More details of the calculations can be found in~\cite{russo2014learning}. As stated in \cite{russo2014learning}, practical implementation of updating statistics involves integrals, which can be evaluated at a discrete grid of points within interval $[0,1]$. The computational cost of updating statistics is $O(K^2n)$ where $n$ is the number of points used in the discretization of $[0,1]$.
\subsection{Complexity of Optimization Problems involved in IDS based Policies}
The following result shows that problem $P_1$ is a convex optimization problem and has a structure in the optimal solution.
\begin{proposition}\label{prop:complexityP1}
The function $\Psi_t:\bpi_t\rightarrow (\bpi_t^T\bdelta_t)^2/(\bpi_t^T\bG_t\bh_t)$ is convex on $\{\bpi_t\in\mathcal{S}|\bpi_t^T\bG_t\bh_t>0\}$. Moreover, there is an optimal solution $\bpi_t^*$ to problem $P_1$ such that $|\{i:\bpi_t^*(i)>0\}|\leq 2$.
\end{proposition}
The proof is an adaption of the proof of Proposition 1 in \cite{russo2014learning} by replacing $\bg_t$ by $\bG_t\bh_t$. Proposition \ref{prop:complexityP1} shows that problem $P_1$ is a convex optimization problem, which can be solved by a standard convex optimization solver. What's more, there exists an optimal solution with support size of at most 2. One can search all the pairs of actions and find the optimal solution by brute force. For each pair, it remains to solve a convex optimization problem with one parameter by closed form. So the computational complexity is $O(K^2)$.

Problems $P_2$ and $P_3$ are linear programming problems, which can be solved efficiently in polynomial time by standard methods. Moreover, the following result shows that they can be solved much faster. The proof is presented in Appendix \ref{proof:complexityLP}
\begin{proposition}\label{prop:complexityLP}
The optimization problems $P_2$ and $P_3$ can be solved in $O(K)$ iterations.
\end{proposition}

In sum, the computational complexity of the proposed IDS based policies (including TS-N) is $O(K^2n)$ per iteration, where $n$ is the number of points used in the discretization of $[0,1]$. Note that the complexity of UCB based policies is $O(K)$. IDS based policies can improve the regret performance with reasonable computation cost.
}
\section{Numerical Results}\label{sec:simul}
This section presents numerical results from experiments that evaluate the effectiveness of IDS based policies in comparison to alternative algorithms. We consider the classical Beta-Bernoulli bandit problem with independent actions. The reward of each action $i$ is a Bernoulli$(\theta_i)$ random variable and $\theta_i$ is independently drawn from Beta$(1,1)$. In the experiment, we set $K=5$ and $T=1000$. All the regret results are averaged over $1000$ trials. 

Figure~\ref{fig:determ} presents the cumulative regret results under the deterministic graph feedback. For the time-invariant case, we use a graph with 2 cliques, presented in Appendix \ref{com:graph}. For the time-variant case, the sequence of graphs is generated by the Erd{\H{o}}s-R{\'e}nyi model\footnote{It is different from the Erd{\H{o}}s-R{\'e}nyi random graph feedback since the graph is revealed to the decision maker before the decision making}. We compare our policies to three other algorithms that are proposed for the stochastic bandit with deterministic graph feedback. Caron et al.~\cite{bhagat} proposed UCB-N and UCB-maxN that closely follow the UCB policy and use side observations for better reward estimates (UCB-N) or choose one of the neighboring nodes with a better empirical estimate (UCB-maxN). It has been shown that the regret of UCB-N and UCB-maxN scale with the clique cover number in the time-invariant case. Buccapatnam et al.~\cite{sigmetrics2014} improved the results in \cite{bhagat} with LP-based algorithms ($\epsilon_t$-greedy-LP and UCB-LP\footnote{The result of UCB-LP is omitted from Figure \ref{fig:determ} because it can not be adapted to the time-variant case. Its regret result is similar to that of $\epsilon_t$-greedy-LP in the time-invariant case.}) and guarantees scaling with the domination number\footnote{Domination number is the smallest cardinality of a dominating set, such that any node not in this set is adjacent to at least a member of this set.} in the time-invariant case. We find that TS-N policy outperforms these three algorithms, which is consistent with the empirical observation in the bandit feedback setting~\cite{chapelle2011empirical}. In addition, IDS-N, IDSN-LP and IDS-LP outperform TS-N policy in both cases. These improvements stem from the exploitation of graph structure in IDS based policies, which raises an open question of determining better regret bounds for our IDS based policies in terms of graph structure.

Figure~\ref{fig:random} presents the cumulative regret results under the Erd{\H{o}}s-R{\'e}nyi random graph feedback. For the time-invariant case, we fix the parameter $r=0.25$. For the time-variant case, the parameter $r_t$ is independently drawn from the uniform distribution over the interval $[0,1]$. We compare our policies to UCB-N\footnote{UCB-N is unaware of the graph structure. So it works under the random graph feedback while UCB-maxN and $\epsilon_t$-greedy-LP do not.} and two other algorithms (Exp3-SET of \cite{alon2013bandits} and Exp3-Res) designed for the non-stochastic bandit with random graph feedback. The average regrets of Exp3-SET and Exp3-Res are dramatically larger than that of IDS based policies. For this reason, parts of Exp3-SET and Exp3-Res are omitted from Figure~\ref{fig:random}. Although, Exp3-SET and Exp3-Res have similar problem-independent upper bounds of regret, our policies utilize the stochastic model and outperform these counterparts. In addition, our IDS based policies outperform TS-N and UCB-N as well. It is interesting that IDS-LP policy performs well in both experiments though it has an upper bound that scales with the number of actions. The reason is that IDS-LP is greedy in minimizing the expected instantaneous regret, however, with guaranteed extent of exploration.

\begin{figure}[t]
\centering
\begin{subfigure}[b]{0.45\textwidth}
    \includegraphics[width=\textwidth]{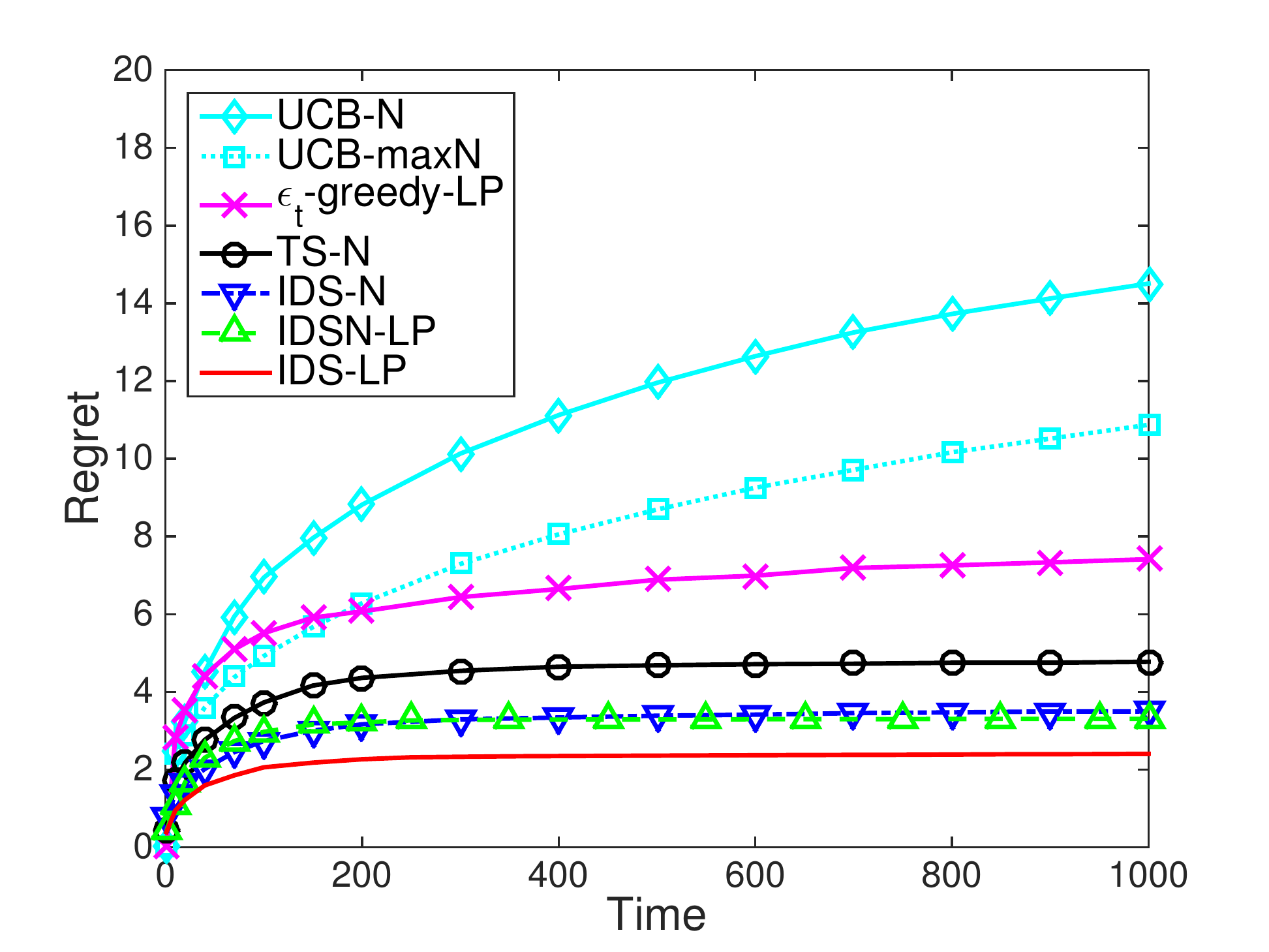}
     \caption{Time-invariant graphs}\label{subfig:determinvar}
\end{subfigure}
\begin{subfigure}[b]{0.45\textwidth}
    \includegraphics[width=\textwidth]{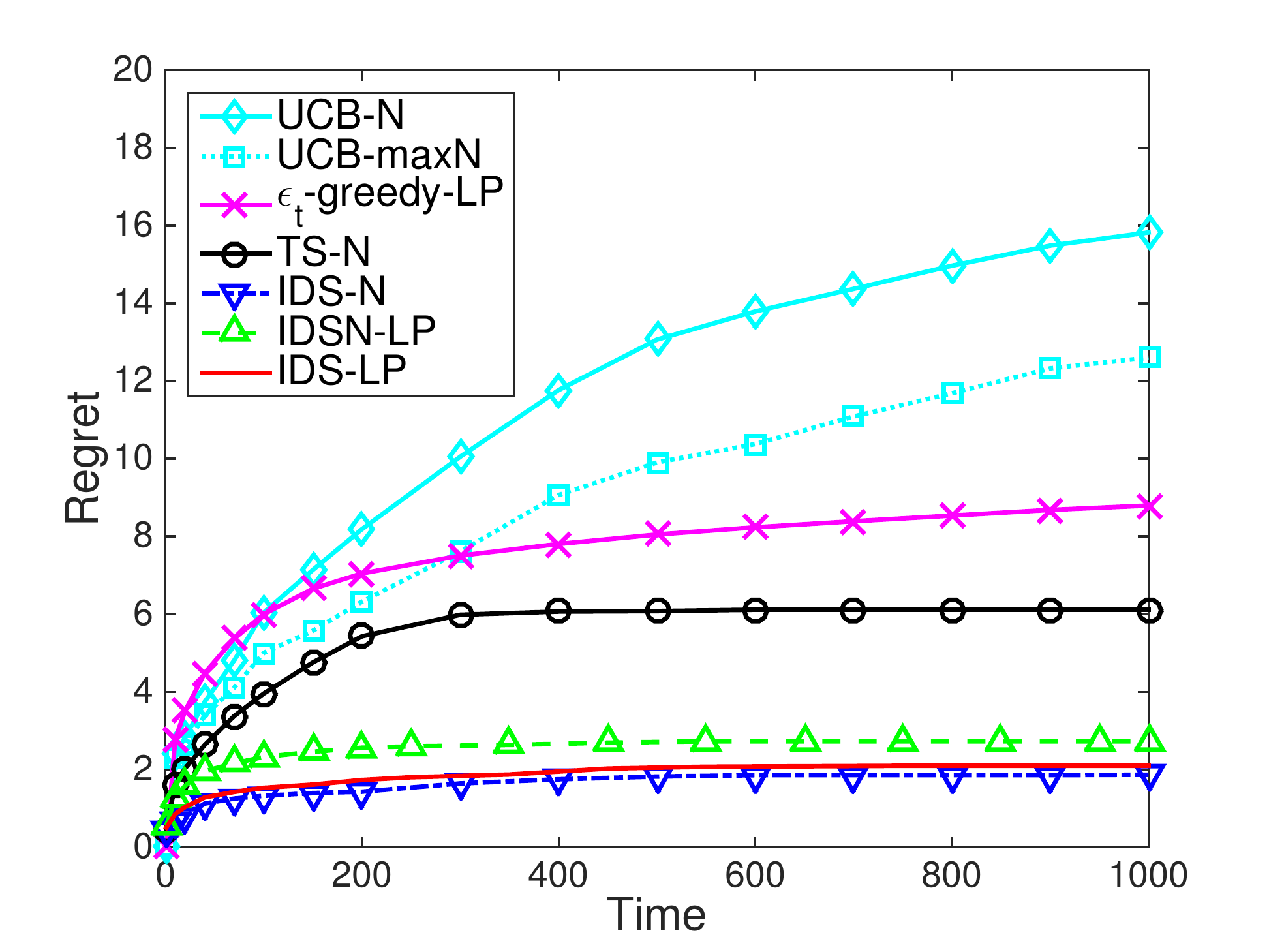}
     \caption{Time-variant graphs}\label{subfig:determvar}
\end{subfigure}
\caption{Regrets under the deterministic graph feedback}
\label{fig:determ}
\end{figure}

\begin{figure}[t]
\centering
\begin{subfigure}[b]{0.45\textwidth}
    \includegraphics[width=\textwidth]{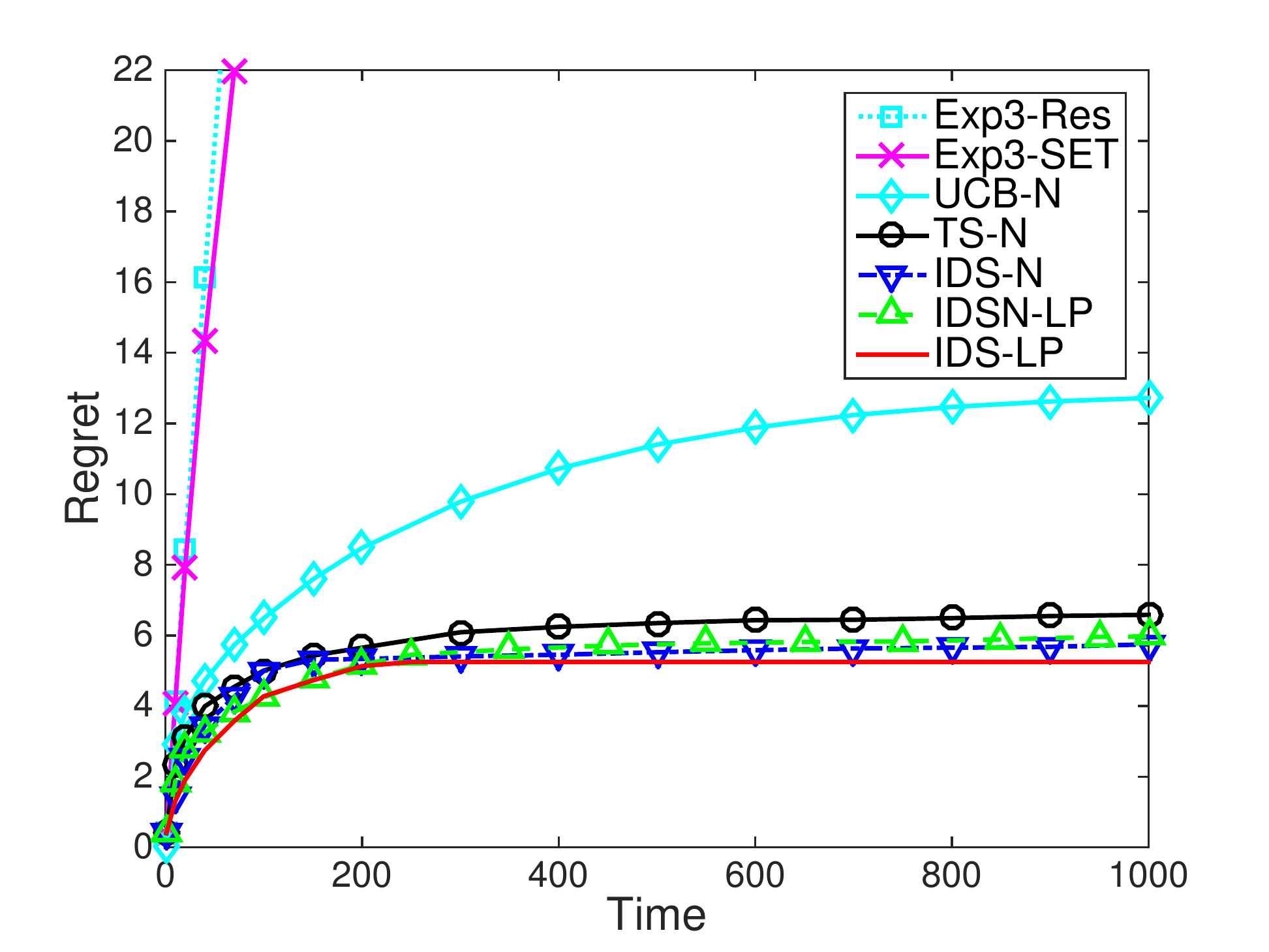}
     \caption{Time-invariant $r=0.25$}\label{subfig:randinvar}
\end{subfigure}
\begin{subfigure}[b]{0.45\textwidth}
    \includegraphics[width=\textwidth]{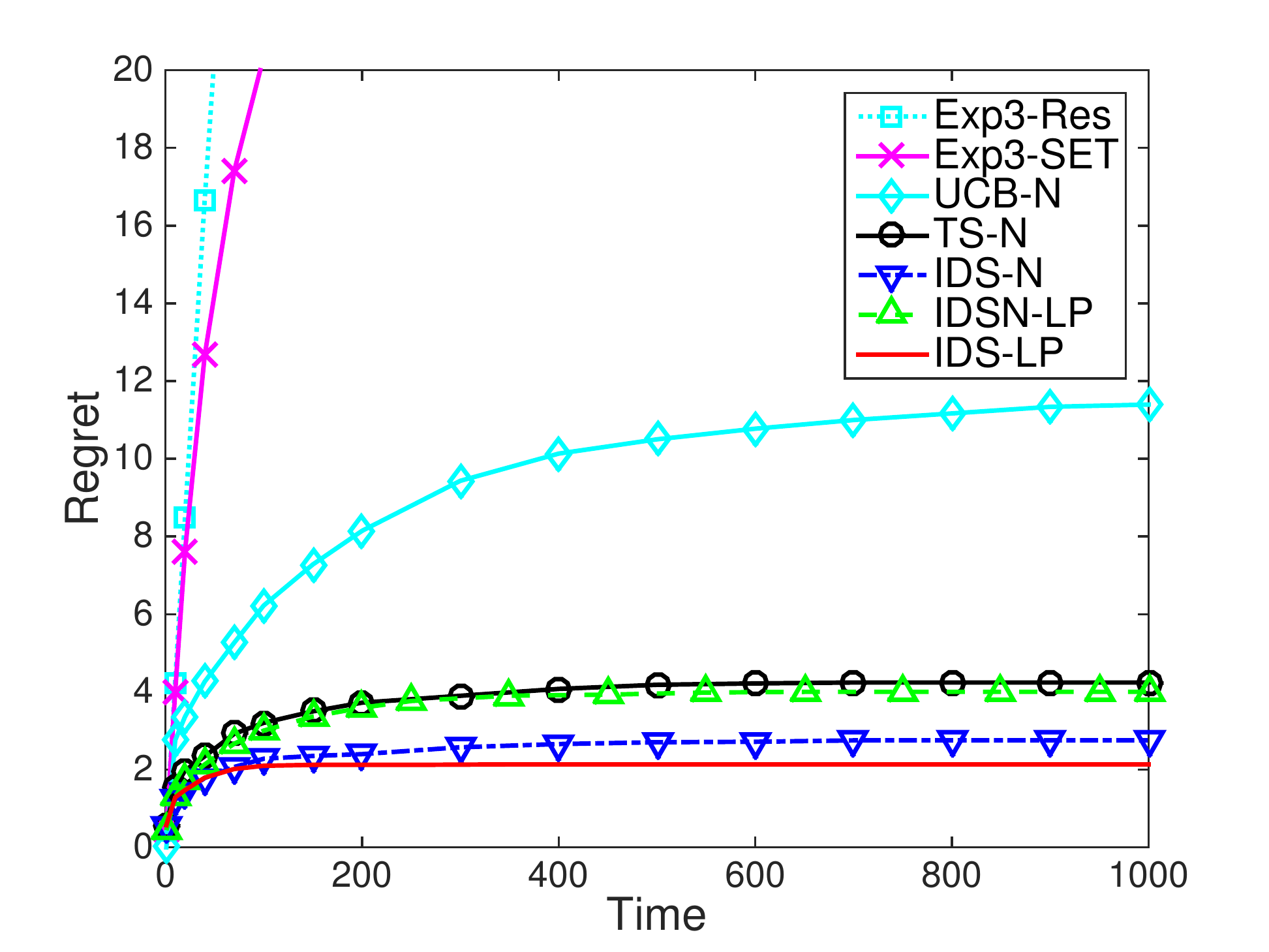}
     \caption{Time-variant $r_t$}\label{subfig:randvar}
\end{subfigure}
\caption{Regrets under the Erd{\H{o}}s-R{\'e}nyi random graph feedback}
\label{fig:random}
\end{figure}
\section{Conclusion}
We have proposed Information Directed Sampling based policies and presented Thompson Sampling for stochastic multi-armed bandits with both deterministic and random graph feedback. We establish a unified Bayesian regret bound, that scales with the clique cover number of the graph, for TS-N, IDS-N and IDSN-LP policies under the deterministic graph case. We also present the first known theoretical guarantee, that scales with the ratio of number of actions over the expected number of observations per iteration, for TS-N, IDS-N and IDSN-LP policies under the random graph case. These results allow us to uncover the gain of partial feedback between the bandit feedback and full information feedback. Finally, we demonstrate state of art performance in numerical experiments.

This work raises the following open questions. It would be interesting to find a problem-independent regret bound that scales with the independence number of the graph instead of the clique cover number for IDS-N and IDSN-LP policies under the deterministic graph case. We believe that such improvement against TS-N can be established by exploiting the graph structure in IDS-N and IDSN-LP policies, as shown in Figure \ref{fig:determ}. Another interesting problem is to find a tighter bound for IDS-LP policy. Intuitively, IDS-LP policy can have low regret due to its greedy nature. Further, it would be an interesting extension to our work to consider a preferential attachment random graph and other growing graphs with time to model the growth process in social networks with time. 
\newpage
\bibliographystyle{aaai}
\bibliography{refs,ref}

\begin{thebibliography}{}

\bibitem[\protect\citeauthoryear{Alon \bgroup et al\mbox.\egroup
  }{2013}]{alon2013bandits}
Alon, N.; Cesa-Bianchi, N.; Gentile, C.; and Mansour, Y.
\newblock 2013.
\newblock From bandits to experts: A tale of domination and independence.
\newblock In {\em Advances in Neural Information Processing Systems},
  1610--1618.

\bibitem[\protect\citeauthoryear{Alon \bgroup et al\mbox.\egroup
  }{2014}]{alon2014nonstochastic}
Alon, N.; Cesa-Bianchi, N.; Gentile, C.; Mannor, S.; Mansour, Y.; and Shamir,
  O.
\newblock 2014.
\newblock Nonstochastic multi-armed bandits with graph-structured feedback.
\newblock {\em arXiv preprint arXiv:1409.8428}.

\bibitem[\protect\citeauthoryear{Alon \bgroup et al\mbox.\egroup
  }{2015}]{alon2015online}
Alon, N.; Cesa-Bianchi, N.; Dekel, O.; and Koren, T.
\newblock 2015.
\newblock Online learning with feedback graphs: Beyond bandits.
\newblock In {\em COLT},  23--35.

\bibitem[\protect\citeauthoryear{Audibert and
  Bubeck}{2010}]{audibert2010regret}
Audibert, J.-Y., and Bubeck, S.
\newblock 2010.
\newblock Regret bounds and minimax policies under partial monitoring.
\newblock {\em Journal of Machine Learning Research} 11(Oct):2785--2836.

\bibitem[\protect\citeauthoryear{Auer, Cesa-Bianchi, and
  Fischer}{2002}]{auer2002finite}
Auer, P.; Cesa-Bianchi, N.; and Fischer, P.
\newblock 2002.
\newblock Finite-time analysis of the multiarmed bandit problem.
\newblock {\em Machine learning} 47(2-3):235--256.

\bibitem[\protect\citeauthoryear{Buccapatnam \bgroup et al\mbox.\egroup
  }{2017}]{buccapatnam2017reward}
Buccapatnam, S.; Liu, F.; Eryilmaz, A.; and Shroff, N.~B.
\newblock 2017.
\newblock Reward maximization under uncertainty: Leveraging side-observations
  on networks.
\newblock {\em arXiv preprint arXiv:1704.07943}.

\bibitem[\protect\citeauthoryear{Buccapatnam, Eryilmaz, and
  Shroff}{2014}]{sigmetrics2014}
Buccapatnam, S.; Eryilmaz, A.; and Shroff, N.~B.
\newblock 2014.
\newblock Stochastic bandits with side observations on networks.
\newblock {\em SIGMETRICS Perform. Eval. Rev.} 42(1):289--300.

\bibitem[\protect\citeauthoryear{Caron \bgroup et al\mbox.\egroup
  }{2012}]{bhagat}
Caron, S.; Kveton, B.; Lelarge, M.; and Bhagat, S.
\newblock 2012.
\newblock Leveraging side observations in stochastic bandits.
\newblock In {\em UAI},  142--151.
\newblock AUAI Press.

\bibitem[\protect\citeauthoryear{Carpentier and
  Valko}{2016}]{carpentier2016revealing}
Carpentier, A., and Valko, M.
\newblock 2016.
\newblock Revealing graph bandits for maximizing local influence.
\newblock In {\em International Conference on Artificial Intelligence and
  Statistics},  10--18.

\bibitem[\protect\citeauthoryear{Chapelle and Li}{2011}]{chapelle2011empirical}
Chapelle, O., and Li, L.
\newblock 2011.
\newblock An empirical evaluation of thompson sampling.
\newblock In {\em Advances in neural information processing systems},
  2249--2257.

\bibitem[\protect\citeauthoryear{Chen \bgroup et al\mbox.\egroup
  }{2016}]{chen2016combinatorial}
Chen, W.; Wang, Y.; Yuan, Y.; and Wang, Q.
\newblock 2016.
\newblock Combinatorial multi-armed bandit and its extension to
  probabilistically triggered arms.
\newblock {\em Journal of Machine Learning Research} 17(50):1--33.

\bibitem[\protect\citeauthoryear{Cohen, Hazan, and Koren}{2016}]{Cohen2016}
Cohen, A.; Hazan, T.; and Koren, T.
\newblock 2016.
\newblock Online learning with feedback graphs without the graphs.
\newblock {\em CoRR} abs/1605.07018.

\bibitem[\protect\citeauthoryear{Koc{\'a}k \bgroup et al\mbox.\egroup
  }{2014}]{kocak2014efficient}
Koc{\'a}k, T.; Neu, G.; Valko, M.; and Munos, R.
\newblock 2014.
\newblock Efficient learning by implicit exploration in bandit problems with
  side observations.
\newblock In {\em Advances in Neural Information Processing Systems},
  613--621.

\bibitem[\protect\citeauthoryear{Koc{\'a}k, Neu, and
  Valko}{2016}]{kocak2016online}
Koc{\'a}k, T.; Neu, G.; and Valko, M.
\newblock 2016.
\newblock Online learning with erd{\H{o}}s-r{\'e}nyi side-observation graphs.
\newblock In {\em Uncertainty in Artificial Intelligence}.

\bibitem[\protect\citeauthoryear{Mannor and Shamir}{2011}]{mannor}
Mannor, S., and Shamir, O.
\newblock 2011.
\newblock From bandits to experts: On the value of side-observations.
\newblock In {\em NIPS},  684--692.

\bibitem[\protect\citeauthoryear{Russo and Roy}{2014}]{journals/corr/RussoR14a}
Russo, D., and Roy, B.~V.
\newblock 2014.
\newblock Learning to optimize via information directed sampling.
\newblock {\em CoRR} abs/1403.5556.

\bibitem[\protect\citeauthoryear{Russo and Van~Roy}{2014}]{russo2014learning}
Russo, D., and Van~Roy, B.
\newblock 2014.
\newblock Learning to optimize via information-directed sampling.
\newblock In {\em Advances in Neural Information Processing Systems},
  1583--1591.

\bibitem[\protect\citeauthoryear{Russo and
  Van~Roy}{2016}]{russo2016information}
Russo, D., and Van~Roy, B.
\newblock 2016.
\newblock An information-theoretic analysis of thompson sampling.
\newblock {\em Journal of Machine Learning Research} 17(68):1--30.

\bibitem[\protect\citeauthoryear{Scott}{2010}]{scott2010modern}
Scott, S.~L.
\newblock 2010.
\newblock A modern bayesian look at the multi-armed bandit.
\newblock {\em Applied Stochastic Models in Business and Industry}
  26(6):639--658.

\bibitem[\protect\citeauthoryear{Seldin \bgroup et al\mbox.\egroup
  }{2014}]{seldin2014prediction}
Seldin, Y.; Bartlett, P.; Crammer, K.; and Abbasi-Yadkori, Y.
\newblock 2014.
\newblock Prediction with limited advice and multiarmed bandits with paid
  observations.
\newblock In {\em International Conference on Machine Learning},  280--287.

\bibitem[\protect\citeauthoryear{Thompson}{1933}]{thompson}
Thompson, W.~R.
\newblock 1933.
\newblock On the likelihood that one unknown probability exceeds another in
  view of the evidence of two samples.
\newblock {\em Biometrika} 25(3/4):285--294.

\bibitem[\protect\citeauthoryear{Tossou, Dimitrakakis, and
  Dubhashi}{2017}]{tossou2017thompson}
Tossou, A.; Dimitrakakis, C.; and Dubhashi, D.
\newblock 2017.
\newblock Thompson sampling for stochastic bandits with graph feedback.
\newblock In {\em AAAI Conference on Artificial Intelligence}.

\bibitem[\protect\citeauthoryear{Ugander \bgroup et al\mbox.\egroup
  }{2011}]{ugander2011Facebook}
Ugander, J.; Karrer, B.; Backstrom, L.; and Marlow, C.
\newblock 2011.
\newblock The anatomy of the facebook social graph.
\newblock {\em arXiv preprint arXiv:1111.4503}.

\bibitem[\protect\citeauthoryear{Valko}{2016}]{valko2016bandits}
Valko, M.
\newblock 2016.
\newblock {\em Bandits on graphs and structures}.
\newblock Ph.D. Dissertation, {\'E}cole normale sup{\'e}rieure de Cachan-ENS
  Cachan.

\bibitem[\protect\citeauthoryear{Wu, Gy{\"o}rgy, and
  Szepesv{\'a}ri}{2015}]{wu2015online}
Wu, Y.; Gy{\"o}rgy, A.; and Szepesv{\'a}ri, C.
\newblock 2015.
\newblock Online learning with gaussian payoffs and side observations.
\newblock In {\em Advances in Neural Information Processing Systems},
  1360--1368.

\end{thebibliography}
\newpage
\onecolumn
\appendix
\section{Proof}
We first present two facts and prove three useful lemmas. Then we provide the proofs of previous results. The following fact is well known in information theory and is shown as Fact 6 in~\cite{russo2016information}.
\begin{fact}\label{fact:mutualinfo}
\emph{(KL divergence form of mutual information)} For a discrete random variable $A$ over finite set $\cK$ and random variable $Y$, the mutual information
\begin{align}
I(A;Y)&\triangleq D(\mathbb{P}((A,Y)\in\cdot)||\mathbb{P}(A\in\cdot)\mathbb{P}(Y\in\cdot))\\
&=\mathbb{E}_{A}[D(\mathbb{P}(Y\in\cdot|A)||\mathbb{P}(Y\in\cdot))]\\
&=\sum_{a\in\cK}\mathbb{P}(A=a)D(\mathbb{P}(Y\in\cdot|A=a)||\mathbb{P}(Y\in\cdot)).
\end{align}
\end{fact}

The following fact, which is Fact 3 in \cite{russo2016information}, shows that the mutual information between $X$ and $Y$ is the expected reduction in the entropy of $X$ due to observing $Y$.
\begin{fact}\label{fact:entropy}
\emph{(Entropy form of mutual information)}
\begin{equation}
I(X;Y)=H(X)-H(X|Y)=H(Y)-H(Y|X)=H(X)+H(Y)-H(X,Y).
\end{equation}
\end{fact}

\begin{lemma}\label{lem:bandit}
For any time $t$, we have that $\frac{\left(\bdelta_t^T\balpha_t\right)^2}{\bh_t^T\balpha_t}\leq \frac{K}{2}$ almost surely.
\end{lemma}
\begin{proof}
By the definition of the instantaneous regret, we have that
\begin{align}
\bdelta_t^T\balpha_t&=\sum_{a\in\cK}\balpha_t(a)\mathbb{E}[Y_{t,A^*}-Y_{t,a}|\mathcal{F}_t]\\
&\overset{(a)}{=}\sum_{a\in\cK}\balpha_t(a)\left(\mathbb{E}[Y_{t,A^*}|\cF_t]-\mathbb{E}[Y_{t,a}|\mathcal{F}_t]\right)\\
&\overset{}{=}\mathbb{E}[Y_{t,A^*}|\cF_t]-\sum_{a\in\cK}\balpha_t(a)\mathbb{E}[Y_{t,a}|\mathcal{F}_t]\\
&\overset{(b)}{=}\sum_{a\in\cK}\balpha_t(a)\left(\mathbb{E}[Y_{t,a}|\cF_t,A^*=a]-\mathbb{E}[Y_{t,a}|\mathcal{F}_t]\right),\label{eqn:instantregret}
\end{align}
where $(a)$ follows from the linearity of expectation, $(b)$ uses the law of total probability.

By the definition of the information gain of observing an action, we have that
\begin{align}
\bh_t^T\balpha_t&=\sum_{a\in\cK}\balpha_t(a)I_t(A^*;Y_{t,a})\\
&\overset{(c)}{=}\sum_{a\in\cK}\balpha_t(a)\left(\sum_{a^*\in\cK}\mathbb{P}(A^*=a^*|\cF_t)D(\mathbb{P}(Y_{t,a}\in\cdot|\cF_t,A^*=a^*)||\mathbb{P}(Y_{t,a}\in\cdot|\cF_t))\right)\\
&=\sum_{a\in\cK}\sum_{a^*\in\cK}\balpha_t(a)\balpha_t(a^*)D(\mathbb{P}(Y_{t,a}\in\cdot|\cF_t,A^*=a^*)||\mathbb{P}(Y_{t,a}\in\cdot|\cF_t)),\label{eqn:instantgain}
\end{align}
where $(c)$ follows from Fact~\ref{fact:mutualinfo}. Now, we are ready to bound the ratio.

\begin{align}
\left(\bdelta_t^T\balpha_t\right)^2&\overset{(d)}{=}\left(\sum_{a\in\cK}\balpha_t(a)\left(\mathbb{E}[Y_{t,a}|\cF_t,A^*=a]-\mathbb{E}[Y_{t,a}|\mathcal{F}_t]\right)\right)^2\\
&\overset{(e)}{\leq} \left(\sum_{a\in\cK} 1^2\right)\left(\sum_{a\in\cK}\balpha_t(a)^2\left(\mathbb{E}[Y_{t,a}|\cF_t,A^*=a]-\mathbb{E}[Y_{t,a}|\mathcal{F}_t]\right)^2\right)\\
&\overset{(f)}{\leq} \frac{K}{2}\sum_{a\in\cK}\balpha_t(a)^2D\left(\mathbb{P}(Y_{t,a}\in\cdot|\cF_t,A^*=a)||\mathbb{P}(Y_{t,a}\in\cdot|\mathcal{F}_t)\right)\\
&\overset{(g)}{\leq} \frac{K}{2}\sum_{a\in\cK}\balpha_t(a)\left(\sum_{a^*\in\cK}\balpha_t(a^*)D(\mathbb{P}(Y_{t,a}\in\cdot|\cF_t,A^*=a^*)||\mathbb{P}(Y_{t,a}\in\cdot|\cF_t))\right)\\
&\overset{(h)}{=} \frac{K}{2}\bh_t^T\balpha_t,
\end{align}
where $(d)$ follows from (\ref{eqn:instantregret}), $(e)$ uses Cauchy-Schwartz inequality, $(f)$ follows from Pinsker's inequality, $(g)$ follows by adding some nonnegative terms and $(h)$ follows from (\ref{eqn:instantgain}).
\end{proof}

\begin{lemma}\label{lem:fullinfo}
For any time $t$, we have that $\frac{\left(\bdelta_t^T\balpha_t\right)^2}{\bh_t^T\boldsymbol{1}}\leq \frac{1}{2}$ almost surely, where $\boldsymbol{1}=(1,\ldots,1)^T\in\mathbb{R}^K$.
\end{lemma}
\begin{proof}
By the definition of the information gain of observing an action, we have that
\begin{align}
\bh_t^T\boldsymbol{1}&=\sum_{a\in\cK}I_t(A^*;Y_{t,a})\\
&\overset{(a)}{=}\sum_{a\in\cK}\left(\sum_{a^*\in\cK}\mathbb{P}(A^*=a^*|\cF_t)D(\mathbb{P}(Y_{t,a}\in\cdot|\cF_t,A^*=a^*)||\mathbb{P}(Y_{t,a}\in\cdot|\cF_t))\right)\\
&=\sum_{a\in\cK}\sum_{a^*\in\cK}\balpha_t(a^*)D(\mathbb{P}(Y_{t,a}\in\cdot|\cF_t,A^*=a^*)||\mathbb{P}(Y_{t,a}\in\cdot|\cF_t)),\label{eqn:instantgain2}
\end{align}
where $(a)$ follows from the Fact~\ref{fact:mutualinfo}. Now, we are ready to bound the ratio.
\begin{align}
\bdelta_t^T\balpha_t&\overset{(b)}{=}\sum_{a\in\cK}\balpha_t(a)\left(\mathbb{E}[Y_{t,a}|\cF_t,A^*=a]-\mathbb{E}[Y_{t,a}|\mathcal{F}_t]\right)\\
&\overset{(c)}{\leq}\sum_{a\in\cK}\balpha_t(a)\sqrt{\frac{1}{2}D\left(\mathbb{P}(Y_{t,a}\in\cdot|\cF_t,A^*=a)||\mathbb{P}(Y_{t,a}\in\cdot|\mathcal{F}_t)\right)}\\
&\overset{(d)}{\leq}\sqrt{\frac{1}{2}\sum_{a\in\cK}\balpha_t(a)D\left(\mathbb{P}(Y_{t,a}\in\cdot|\cF_t,A^*=a)||\mathbb{P}(Y_{t,a}\in\cdot|\mathcal{F}_t)\right)}\\
&\overset{(e)}{\leq}\sqrt{\frac{1}{2}\sum_{a\in\cK}\sum_{a^*\in\cK}\balpha_t(a^*)D(\mathbb{P}(Y_{t,a}\in\cdot|\cF_t,A^*=a^*)||\mathbb{P}(Y_{t,a}\in\cdot|\cF_t))}\\
&\overset{(f)}{=}\sqrt{\frac{1}{2}\bh_t^T\boldsymbol{1}},
\end{align}
where $(b)$ follows from (\ref{eqn:instantregret}), $(c)$ follows from Pinsker's inequality, $(d)$ follows from Jensen's inequality, $(e)$ follows from adding some nonnegative terms and $(f)$ follows from (\ref{eqn:instantgain2}).
\end{proof}

\begin{lemma}\label{lem:addindpt}
For any random variable $A$ and integer $n$, if $X_1,\ldots,X_n$ are mutually independent, then the mutual information
\begin{equation}
I(A;X_1,\ldots,X_n)\geq\sum_{i=1}^nI(A;X_i).
\end{equation}
\end{lemma}
\begin{proof}
First, we consider the case where $n=2$. By Fact \ref{fact:entropy}, we have that
\begin{align}
&I(A;X_1)+I(A;X_2)-I(A;X_1,X_2)\\
=&H(X_1)-H(X_1|A)+H(X_2)-H(X_2|A)-\left(H(X_1,X_2)-H(X_1,X_2|A)\right)\\
=&H(X_1)+H(X_2)-H(X_1,X_2)-(H(X_1|A)+H(X_2|A)-H(X_1,X_2|A))\\
=&I(X_1;X_2)-I(X_1;X_2|A)\\
\overset{(a)}{=}&-I(X_1;X_2|A)\\
\overset{(b)}{\leq}& 0,
\end{align}
where $(a)$ follows from $I(X_1;X_2)=0$ since $X_1$ and $X_2$ are independent and $(b)$ uses that mutual information is nonnegative. The result follows from iteratively applying the above result for $n-1$ times.
\end{proof}

\subsection{Proof of Proposition \ref{prop:infogain}}\label{proof:infogain}
\begin{proof}
Note that $\bg_t\geq\bG_t\bh_t$ is equivalent to $\bg_t(i)\geq\sum_{a\in\cK}\bG_t(i,a)\bh_t(a)$, $\forall i\in\cK$. Now, fix any $i\in\cK$. Recall that $G_t=(\cK,\cE_t)$ is the (deterministic or random) feedback graph. Let $\cN_t(i)=\{a\in\cK|(i,a)\in\cE_t\}$ be the set of actions that can be observed when playing action $i$. The intuition behind the proof is that the entropy reduction by playing action $i$ is equivalent to the mutual information between $A^*$ and all observations $(Y_{t,a})_{a\in\cN_t(i)}$. Formally, by the definition of $\bg_t(i)$ we have that
\begin{align}
\bg_t(i)&=\mathbb{E}[H(\balpha_t)-H(\balpha_{t+1})|\mathcal{F}_t,A_t=i]\\
&\overset{(a)}{=}I_t(A^*;(Y_{t,a})_{a\in\cN_t(i)})\\
&\overset{(b)}{\geq}\sum_{a\in\cN_t(i)}I_t(A^*;Y_{t,a})\\
&\overset{(c)}{=}\sum_{a\in\cN_t(i)}\bh_t(a),\label{eqn:infogain}
\end{align}
where $(a)$ follows from Fact \ref{fact:entropy}, $(b)$ follows from Lemma \ref{lem:addindpt} and that all actions are mutually independent and $(c)$ follows from the definition of $\bh_t(a)$.

Under the deterministic graph, we have that $\sum_{a\in\cN_t(i)}\bh_t(a)=\sum_{a\in\cK}\bG_t(i,a)\bh_t(a)$ by the definition of $\cN_t(i)$. By (\ref{eqn:infogain}) we have $\bg_t(i)\geq\sum_{a\in\cK}\bG_t(i,a)\bh_t(a)$ under the deterministic graph.

Under the random graph, (\ref{eqn:infogain}) holds given the graph $G_t$. By the law of total expectation and (\ref{eqn:infogain}), we have that $\bg_t(i)\geq\sum_{a\in\cK}\mathbb{P}\left((i,a)\in\cE_t\right)\bh_t(a)=\sum_{a\in\cK}\bG_t(i,a)\bh_t(a)$.

In sum, we show that $\bg_t(i)\geq\sum_{a\in\cK}\bG_t(i,a)\bh_t(a)$ for any $i\in\cK$ under any type of graph. The result follows.
\end{proof}

\subsection{Proof of Proposition \ref{prop:ratiobound}}\label{proof:ratiobound}
\begin{proof}
First, we bound the information ratio $\Psi_t(\bpi_t^{\text{TS-N}})$ by Proposition \ref{prop:infogain}
\begin{equation}
\Psi_t(\bpi_t^{\text{TS-N}})=\frac{\left(\bdelta_t^T\bpi_t^{\text{TS-N}}\right)^2}{\bg_t^T\bpi_t^{\text{TS-N}}}\leq\frac{\left(\bdelta_t^T\bpi_t^{\text{TS-N}}\right)^2}{(\bG_t\bh_t)^T\bpi_t^{\text{TS-N}}}=\frac{\left(\bdelta_t^T\balpha_t\right)^2}{(\bG_t\bh_t)^T\balpha_t}=\psi_t
\end{equation}

Second, we bound the information ratio $\Psi_t(\bpi_t^{\text{IDS-N}})$ by Proposition \ref{prop:infogain}
\begin{equation}
\Psi_t(\bpi_t^{\text{IDS-N}})=\frac{\left(\bdelta_t^T\bpi_t^{\text{IDS-N}}\right)^2}{\bg_t^T\bpi_t^{\text{IDS-N}}}\leq\frac{\left(\bdelta_t^T\bpi_t^{\text{IDS-N}}\right)^2}{(\bG_t\bh_t)^T\bpi_t^{\text{IDS-N}}}\overset{(a)}{\leq}\frac{\left(\bdelta_t^T\balpha_t\right)^2}{(\bG_t\bh_t)^T\balpha_t}=\psi_t,
\end{equation}
where $(a)$ uses that $\balpha_t$ is feasible for the problem $P_1$.

Now, we are ready to bound the information ratio $\Psi_t(\bpi_t^{\text{IDSN-LP}})$. By the definition and Proposition \ref{prop:infogain}, we have that
\begin{align}
\Psi_t(\bpi_t^{\text{IDSN-LP}})\leq\frac{\left(\bdelta_t^T\bpi_t^{\text{IDSN-LP}}\right)^2}{(\bG_t\bh_t)^T\bpi_t^{\text{IDSN-LP}}}\overset{(b)}{\leq}\frac{\left(\bdelta_t^T\bpi_t^{\text{IDSN-LP}}\right)^2}{(\bG_t\bh_t)^T\balpha_t}
\overset{(c)}{\leq}\frac{\left(\bdelta_t^T\balpha_t\right)^2}{(\bG_t\bh_t)^T\balpha_t}=\psi_t,
\end{align}
where $(b)$ follows from the constraint of the problem $P_2$ and $(c)$ uses that $\balpha_t$ is feasible for the problem $P_2$.
\end{proof}

\subsection{Proof of Theorem \ref{thm:IDSLP}}\label{proof:IDSLP}
\begin{proof}
By the general bound result in Lemma \ref{prop:generalbound}, what remains is to bound the information ratio $\Psi_t(\bpi_t^{\text{IDS-LP}})$. By the definition and Proposition \ref{prop:infogain}, we have that
\begin{align}
\Psi_t(\bpi_t^{\text{IDS-LP}})&=\frac{\left(\bdelta_t^T\bpi_t^{\text{IDS-LP}}\right)^2}{\bg_t^T\bpi_t^{\text{IDS-LP}}}\leq\frac{\left(\bdelta_t^T\bpi_t^{\text{IDS-LP}}\right)^2}{(\bG_t\bh_t)^T\bpi_t^{\text{IDS-LP}}} \overset{(a)}{\leq} \frac{\left(\bdelta_t^T\bpi_t^{\text{IDS-LP}}\right)^2}{\bh_t^T\balpha_t} \overset{(b)}{\leq} \frac{\left(\bdelta_t^T\balpha_t\right)^2}{\bh_t^T\balpha_t} \overset{(c)}{\leq} \frac{K}{2},
\end{align}
where $(a)$ follows from the constraint of problem $P_3$, $(b)$ uses the fact that $\balpha_t$ is feasible for problem $P_3$ and $(c)$ follows from Lemma \ref{lem:bandit}. Note that the result holds almost surely for any $t$. The regret result follows.
\end{proof}

\subsection{Proof of Theorem \ref{thm:TSNdeterm}}\label{proof:TSNdeterm}
A similar result has been shown in~\cite{tossou2017thompson}, where they extend the concept of an action to an equivalence class (i.e., clique). However, the expectation, KL divergence and mutual information after the extension are lack of careful definition, which makes the proof hard to follow. Here, we provide an alternative proof with details.  
\begin{proof}
By the general bound result in Lemma \ref{prop:generalbound}, it remains to bound the information ratio for each algorithm. By Proposition \ref{prop:ratiobound}, we can simply bound the ratio $\psi_t$ and obtain a unified result. Fix a smallest clique cover, $\mathcal{C}_t$, of the graph $G_t$ such that $|\mathcal{C}_t|=\chi(G_t)$. Let $\bc$ be a clique element of $\mathcal{C}_t$. As shorthand, we let $\balpha_t(\bc)=\sum_{a\in\bc}\balpha_t(a)$. We have that
\begin{align}
(\bG_t\bh_t)^T\balpha_t&=\sum_{a\in\cK}\balpha_t(a)\sum_{a'\in\cK}\bG_t(a,a')I_t(A^*;Y_{t,a'})\\
&\overset{(a)}{=}\sum_{\bc\in\mathcal{C}_t}\sum_{a\in\bc}\balpha_t(a)\sum_{a'\in\cK}\bG_t(a,a')I_t(A^*;Y_{t,a'})\\
&\overset{(b)}{\geq}\sum_{\bc\in\mathcal{C}_t}\sum_{a\in\bc}\balpha_t(a)\sum_{a'\in\bc}I_t(A^*;Y_{t,a'})\\
&=\sum_{\bc\in\mathcal{C}_t}\balpha_t(\bc)\sum_{a\in\bc}I_t(A^*;Y_{t,a})\\
&\overset{(c)}{=}\sum_{\bc\in\mathcal{C}_t}\balpha_t(\bc)\sum_{a\in\bc}\sum_{a^*\in\cK}\balpha_t(a^*)D(\mathbb{P}(Y_{t,a}\in\cdot|\cF_t,A^*=a^*)||\mathbb{P}(Y_{t,a}\in\cdot|\cF_t)),\label{eqn:instantgain3}
\end{align}
where $(a)$ uses that $\mathcal{C}_t$ is a clique cover, $(b)$ follows from the fact that mutual information are nonnegative and $(c)$ follows from the Fact~\ref{fact:mutualinfo}. After bounding the expected information gain, we show an useful result that will be used in bounding the expected instantaneous regret. 
\begin{align}
&\sum_{a\in\bc}\frac{\balpha_t(a)}{\balpha_t(\bc)}\left(\mathbb{E}[Y_{t,a}|\cF_t,A^*=a]-\mathbb{E}[Y_{t,a}|\mathcal{F}_t]\right)\\
\overset{(d)}{\leq}&\sum_{a\in\bc}\frac{\balpha_t(a)}{\balpha_t(\bc)}\sqrt{\frac{1}{2}D\left(\mathbb{P}(Y_{t,a}\in\cdot|\cF_t,A^*=a)||\mathbb{P}(Y_{t,a}\in\cdot|\mathcal{F}_t)\right)}\\
\overset{(e)}{\leq}&\sqrt{\frac{1}{2}\sum_{a\in\bc}\frac{\balpha_t(a)}{\balpha_t(\bc)}D\left(\mathbb{P}(Y_{t,a}\in\cdot|\cF_t,A^*=a)||\mathbb{P}(Y_{t,a}\in\cdot|\mathcal{F}_t)\right)},\label{eqn:generalpinsker}
\end{align}
where $(d)$ follows from Pinsker's inequality and $(e)$ follows from Jensen's inequality. Then we have that
\begin{align}
\bdelta_t^T\balpha_t&\overset{(f)}{=}\sum_{a\in\cK}\balpha_t(a)\left(\mathbb{E}[Y_{t,a}|\cF_t,A^*=a]-\mathbb{E}[Y_{t,a}|\mathcal{F}_t]\right)\\
&\overset{(g)}{=}\sum_{\bc\in\mathcal{C}_t}\balpha_t(\bc)\sum_{a\in\bc}\frac{\balpha_t(a)}{\balpha_t(\bc)}\left(\mathbb{E}[Y_{t,a}|\cF_t,A^*=a]-\mathbb{E}[Y_{t,a}|\mathcal{F}_t]\right)\\
&\overset{(h)}{\leq}\sqrt{\left(\sum_{\bc\in\mathcal{C}_t}1^2\right)\left(\sum_{\bc\in\mathcal{C}_t}\balpha_t(\bc)^2\left(\sum_{a\in\bc}\frac{\balpha_t(a)}{\balpha_t(\bc)}\left(\mathbb{E}[Y_{t,a}|\cF_t,A^*=a]-\mathbb{E}[Y_{t,a}|\mathcal{F}_t]\right)\right)^2\right)}\\
&\overset{(i)}{\leq}\sqrt{\frac{|\mathcal{C}_t|}{2}\sum_{\bc\in\mathcal{C}_t}\balpha_t(\bc)^2\sum_{a\in\bc}\frac{\balpha_t(a)}{\balpha_t(\bc)}D\left(\mathbb{P}(Y_{t,a}\in\cdot|\cF_t,A^*=a)||\mathbb{P}(Y_{t,a}\in\cdot|\mathcal{F}_t)\right)}\\
&=\sqrt{\frac{|\mathcal{C}_t|}{2}\sum_{\bc\in\mathcal{C}_t}\balpha_t(\bc)\sum_{a\in\bc}\balpha_t(a)D\left(\mathbb{P}(Y_{t,a}\in\cdot|\cF_t,A^*=a)||\mathbb{P}(Y_{t,a}\in\cdot|\mathcal{F}_t)\right)}\\
&\overset{(j)}{\leq}\sqrt{\frac{|\mathcal{C}_t|}{2}\sum_{\bc\in\mathcal{C}_t}\balpha_t(\bc)\sum_{a\in\bc}\sum_{a^*\in\cK}\balpha_t(a^*)D\left(\mathbb{P}(Y_{t,a}\in\cdot|\cF_t,A^*=a^*)||\mathbb{P}(Y_{t,a}\in\cdot|\mathcal{F}_t)\right)}\\
&\overset{(k)}{\leq}\sqrt{\frac{|\mathcal{C}_t|}{2}(\bG_t\bh_t)^T\balpha_t},
\end{align}
where $(f)$ follows from (\ref{eqn:instantregret}), $(g)$ uses that $\mathcal{C}_t$ is a clique cover, $(h)$ follows from Cauchy-Schwartz inequality, $(i)$ follows from (\ref{eqn:generalpinsker}), $(j)$ follows by adding some nonnegative terms and $(k)$ follows from (\ref{eqn:instantgain3}). Then we bound the ratio by
\begin{equation}
\psi_t=\frac{\left(\bdelta_t^T\balpha_t\right)^2}{(\bG_t\bh_t)^T\balpha_t}\leq\frac{|\mathcal{C}_t|}{2}=\frac{\chi(\bG_t)}{2}.
\end{equation}

Hence, the result follows from Lemma~\ref{prop:generalbound} and Proposition~\ref{prop:ratiobound}.
\end{proof}

\subsection{Proof of Theorem \ref{thm:TSNrandom}}\label{proof:TSNrandom}
\begin{proof}
By the general bound result in Lemma \ref{prop:generalbound}, what remains is to bound the information ratio for each algorithm. By Proposition \ref{prop:ratiobound}, we can simply bound the ratio $\psi_t$ and obtain a unified result. Recall that $\bG_t(i,i)=1$ for any $i\in\cK$ and $\bG_t(i,j)=r_t$ for any $i\neq j$ under the Erd{\H{o}}s-R{\'e}nyi random graph. Then we have that
\begin{equation}\label{eqn:galpha}
\bG_t^T\balpha_t=(1-r_t)\balpha_t+r_t\boldsymbol{1}.
\end{equation}

Then, we have that 
\begin{align}
\psi_t&=\frac{\left(\bdelta_t^T\balpha_t\right)^2}{(\bG_t\bh_t)^T\balpha_t}\overset{(a)}{=}\frac{\left(\bdelta_t^T\balpha_t\right)^2}{\bh_t^T((1-r_t)\balpha_t+r_t\boldsymbol{1})}\\
&\overset{(b)}{\leq}\frac{\left(\bdelta_t^T\balpha_t\right)^2}{(1-r_t)\frac{2}{K}\left(\bdelta_t^T\balpha_t\right)^2+2r_t\left(\bdelta_t^T\balpha_t\right)^2}=\frac{K}{2(Kr_t+1-r_t)},
\end{align}
where $(a)$ follows from (\ref{eqn:galpha}) and $(b)$ follows from Lemma \ref{lem:bandit} and \ref{lem:fullinfo}.

Hence, the result follows from Lemma~\ref{prop:generalbound} and Proposition~\ref{prop:ratiobound}.
\end{proof}
\subsection{Proof of Proposition \ref{prop:complexityLP}}\label{proof:complexityLP}
\begin{proof}
Note that $\balpha_t$ is a feasible solution for both $P_2$ and $P_3$. The objective function $\bpi_t^T\bdelta_t$ is bounded since each coordinate is bounded by 1. So the optimal value is attained on the vertices of the polyhedra defined by the constraints. Note that the constraints consist of probability simplex and a closed half-space. So the polyhedra has at most $K+1$ vertices. Then the linear program can be solved by visiting all the vertices in $O(K)$ iterations.
\end{proof}

\section{Graph Structure for the Numerical Experiment}\label{com:graph}
Figure \ref{fig:determinvargraph} represents the graph structure we use for the time-invariant and deterministic case. It is clear that the clique cover number is 2 and the independence number is 2 and the domination number is 1. By Theorem \ref{thm:TSNdeterm}, the regrets of TS-N, IDS-N and IDSN-LP scale with 2 instead of 5. Note that for the $\epsilon_t$-greed-LP algorithm of \cite{sigmetrics2014}, it will play action $3$ for exploration since the dominating set is $\{3\}$.
\begin{figure}[hb]
  \centering
    \includegraphics[width=0.3\textwidth]{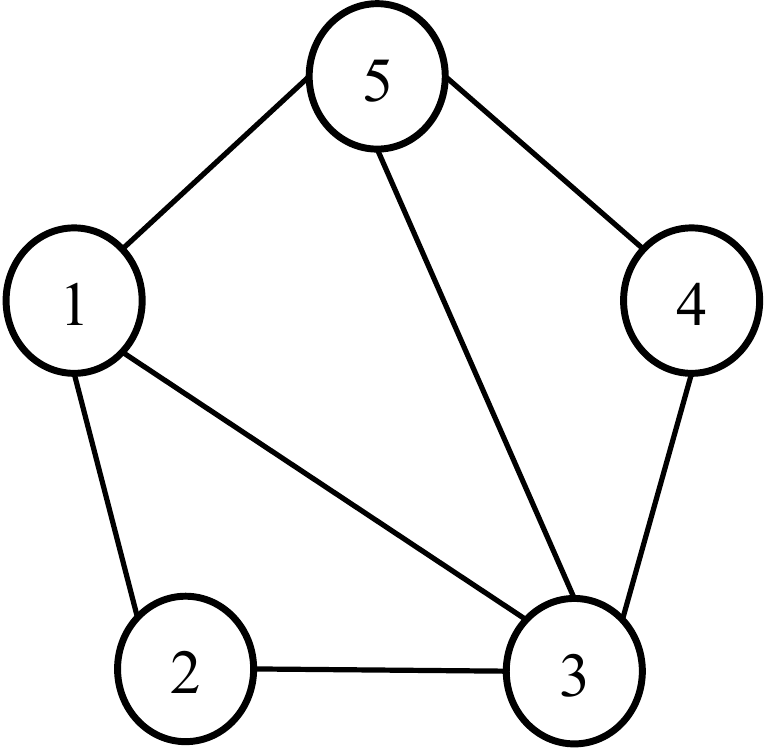}
     \caption{Graph structure for the experiment under time-invariant and deterministic graph}
     \label{fig:determinvargraph}
\end{figure}
\end{document}